\documentclass[12pt]{article}

\usepackage{amsmath, amsthm, amssymb, color}
\usepackage{graphicx}
\usepackage[all]{xypic}
\usepackage{makecell}
\usepackage{mathtools}

\usepackage{tikz}
\oddsidemargin \evensidemargin
\newtheorem{theorem}{Theorem}
\numberwithin{theorem}{section}
\newtheorem{proposition}[theorem]{Proposition}
\newtheorem{lemma}[theorem]{Lemma}
\newtheorem{corollary}[theorem]{Corollary}

\newtheorem{remark}[theorem]{Remark}
\newtheorem{example}[theorem]{Example}

\newcommand{\RR}{\mathbb{R}}

\newcommand{\PP}{\mathbb{P}}

\newcommand{\Mcal}{\mathcal{M}}
\newcommand{\Pcal}{\mathcal{P}}

\newcommand{\Lcal}{\mathcal{L}}
\newcommand{\supp}{\operatorname{supp}}
\newcommand{\RBM}{\operatorname{RBM}}

\usepackage{blkarray}

 \date{}

\graphicspath{{./Figures/}}

\title{\textbf{Mixtures and products in two graphical~models}}

\author{Anna Seigal and Guido Mont\'ufar}

\begin{document}

\maketitle

\begin{abstract} \noindent
We compare two statistical models of three binary random variables. 
One is a mixture model and the other is a product of mixtures model called a restricted Boltzmann machine. 
Although the two models we study look different from their parametrizations, we show that they represent the same set of distributions on the interior of the probability simplex, and are equal up to closure. 
We give a semi-algebraic description of the model in terms of six binomial inequalities and obtain closed form expressions for the maximum likelihood estimates. 
We briefly discuss extensions to larger models. 
\end{abstract}

\section{Introduction}

Graphical models are a popular tool for representing multivariate probability distributions in terms of conditional independence relations (see e.g. \cite{lauritzen1996graphical,Bishop:2006:PRM:1162264}). 
Any probability distribution can be modeled by a graphical model.
However, certain graphs involve many more parameters than others to represent specific distributions. 
In the interest of concisely representing data and reducing computational costs, one would like to understand which structures best represent data. 
For example, deep architectures, with several layers of hidden variables, have become increasingly important in machine learning (see~\cite{Bengio:2009:LDA:1658423.1658424} and references therein). 
Following \cite{montufar2015does} (and using their notation) we focus on two important building blocks to such multi-layer architectures: 
\begin{enumerate}
\item
One hidden variable with $k$ states, connected to $n$ observed binary variables. 
This is the {\em mixture of products} model $\mathcal{M}_{n,k}$. 
Up to scaling, it consists of $2 \times \cdots \times 2$ ($n$ times) tensors of non-negative rank at most $k$, 
$$ p = \sum_{i = 1}^k a_i \otimes b_i \otimes \cdots \otimes c_i, \qquad a_i, b_i, \ldots, c_i \in \mathbb{R}_{\geq 0}^2 .$$
\item 
A layer of $m$ hidden binary variables, each connected to $n$ observed binary variables. 
This is the {\em restricted Boltzmann machine} (RBM) model $\RBM_{n,m}$, also called the
{\em product of mixtures of products} model. 
Up to scale, it consists of $2 \times \cdots \times 2$ ($n$ times) tensors that are the Hadamard product of $m$ tensors of non-negative rank at most two, 
\begin{equation}
p = \prod_{i = 1}^m ( a_i \otimes b_i \otimes\cdots\otimes c_i + d_i \otimes e_i \otimes \cdots\otimes f_i ) , \qquad a_i, b_i,\ldots,  c_i, d_i, e_i,\ldots, f_i \in \mathbb{R}_{\geq 0}^2 .
\label{eq:rbmparametrization}
\end{equation}
\end{enumerate} 
Our main contribution is to find the set of distributions that these models can represent for the first open case $n = 3$. 
We find the semi-algebraic subset of the simplex that the models occupy. 
In doing so, we solve questions posed in~\cite{montufar2015does}. 

The implicit description of a statistical model gives a membership test for distributions, allows the computation of distances to the model (e.g., in terms of the Kullback-Leibler divergence), 
and suggests model-specific algorithms for parameter estimation~\cite{Su,Z}. 
In the above definitions, we consider the polynomial parametrization of the models. They are often defined as marginals of exponential families.\footnotemark
 The two definitions are equivalent up to closure, see for example \cite[Proposition 2.3]{montufar2015does}. 
In contrast to the exponential parametrization, we allow zeros in the decomposition, excluding the possibility that $p$ is identically zero. 

\footnotetext{As marginals of exponential families, RBMs and mixtures of products are given by $p(x) = \frac{1}{Z(W,b,c)}\sum_{y\in\{0,1\}^m} \exp(y^\top W x + c^\top y + b^\top x)$ and $p(x) = \frac{1}{Z(W,b,c)} \sum_{y\in \{e_j\colon j=1,\ldots, k\}  } \exp( y^\top W x + c^\top y + b^\top x)$, respectively, where $x\in\{0,1\}^n$ and $Z(W,b,c)$ is a normalization function. }

We note that $\mathcal{M}_{n,1}$ is the \emph{independence model}, 
described by the intersection of the Segre variety ${\rm Seg}(\PP^1 \times \cdots \times \PP^1)$ with the probability simplex $\Delta_{2^n -1}$ of joint probability distributions of $n$ binary random variables. 
Also, by definition, $\mathcal{M}_{n,2} = \RBM_{n,1}$. 
In \cite{ALLMAN201537} the description of $\mathcal{M}_{n,2}$ is found. The authors describe the `formidable obstacles' to extending their results to hidden variables with more than two states. 
\bigskip

Three binary variables take joint states in $\{ 0, 1\}^3$. 
The $2 \times 2 \times 2$ tensor $(p_{ijk})_{0 \leq i, j, k \leq 1}$ stores the probabilities of these elementary events. Such probability distributions lie in the simplex $\Delta_{2^3 - 1} = \Delta_7$. 
Strictly positive distributions lie on the interior of the simplex. 
We obtain the following description of $\RBM_{3,2}$. 

\begin{theorem} \label{prop}
The statistical model $\RBM_{3,2}$ is described on the interior of the simplex $\Delta_7$ by the union of six basic semi-algebraic sets. One is given by the two inequalities
\begin{equation}
\label{ineq3} \{ p_{000} p_{011} \geq p_{001} p_{010} , \quad p_{100} p_{111} \geq p_{101} p_{110} \} . \end{equation}
The other five
are obtained by permuting indices, and/or reversing the inequalities:
$$ \begin{matrix} 
 \{ p_{000} p_{011} \leq p_{001} p_{010} , \quad p_{100} p_{111} \leq p_{101} p_{110} \} \phantom{.}\\ 
  \{ p_{000} p_{101} \geq p_{001} p_{100} , \quad p_{010} p_{111} \geq p_{011} p_{110} \} \phantom{.}\\
   \{ p_{000} p_{101} \leq p_{001} p_{100} , \quad p_{010} p_{111} \leq p_{011} p_{110} \} \phantom{.}\\
    \{ p_{000} p_{110} \geq p_{100} p_{010} , \quad p_{001} p_{111} \geq p_{101} p_{011} \} \phantom{.} \\
    \{ p_{000} p_{110} \leq p_{100} p_{010} , \quad p_{001} p_{111} \leq p_{101} p_{011} \} . \end{matrix} $$
\end{theorem}

These binomial inequalities correspond to 
determinants of slices of the tensor $(p_{ijk})$. They record conditional correlations in the distribution. 

We compare $\RBM_{3,2}$ to the mixture model $\Mcal_{3,3}$ of non-negative rank at most three tensors. 
Both models are over-parametrized in the seven-dimensional simplex $\Delta_7$, since they have 11 parameters. In \cite{montufar2015does}, it is shown that $\mathcal{M}_{3,3}$ does not fill the simplex. The authors state `we believe that $\mathcal{M}_{3,3}$ and $\RBM_{3,2}$ are very similar, if not equal.' We resolve this question as follows. 

\begin{theorem} \label{conj:eq}
We have the equality $\mathcal{M}_{3,3} = \overline{\RBM_{3,2}}$. Equality $\mathcal{M}_{3,3} = \RBM_{3,2}$ holds on the interior of the simplex. 
\end{theorem} 
\begin{figure}[h]
\centering
\includegraphics[width=7cm]{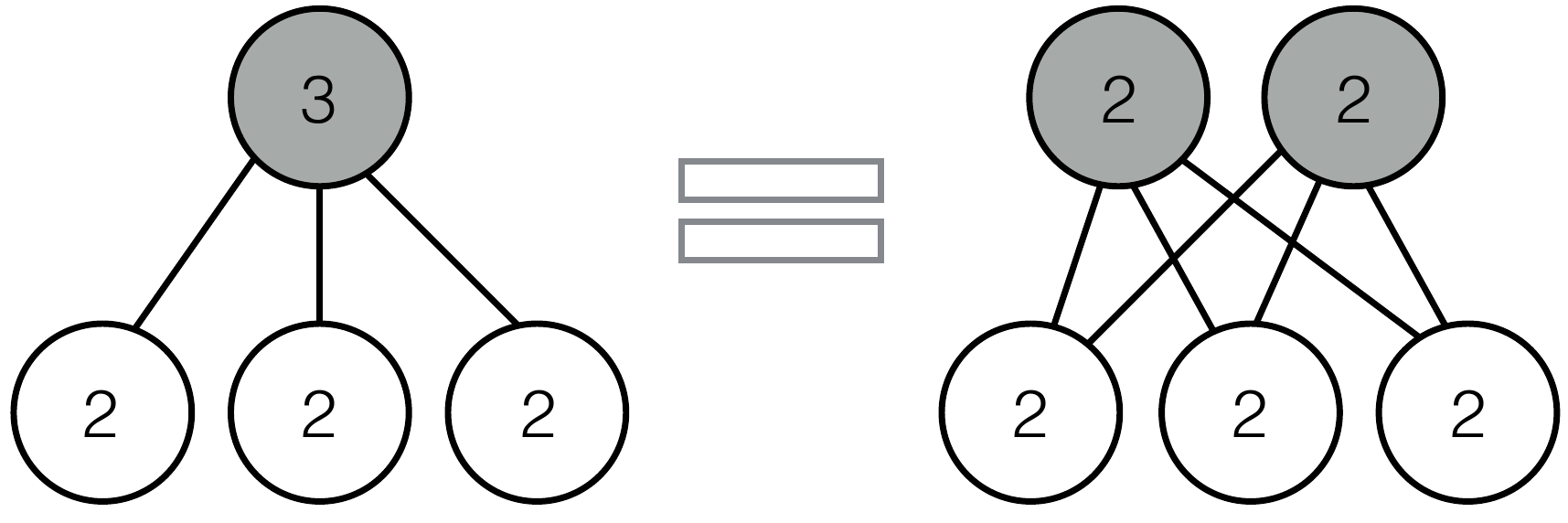}
\caption{A pictorial representation of Theorem \ref{conj:eq}. The label of a variable is the number of states it has; the shaded nodes are hidden.}
\end{figure}

The notation $\overline{\RBM_{3,2}}$ refers to the topological closure of $\RBM_{3,2}$. The mixture model $\Mcal_{3,3}$ and the RBM model $\RBM_{3,2}$ look quite different in their parametrization, but this result shows that they turn out to parametrize the same probability distributions (up to closure). The parametrization of $\RBM_{3,2}$ in~\eqref{eq:rbmparametrization} does not describe a closed set on the boundary of the simplex. 
We describe $\RBM_{3,2}$ on the boundary of the simplex in Proposition \ref{prop:bd}. 
On the other hand, $\mathcal{M}_{3,3}$ is closed (see Proposition~\ref{BW}) and we have the following corollary. 

\begin{corollary} \label{cor:m33}
The model $\mathcal{M}_{3,3}$ is described 
on $\Delta_7$ by the inequalities in Theorem \ref{prop}. 
\end{corollary}

Previous results showed that $\Mcal_{3,3}$ has relative volume at most $96.4\%$, and $\RBM_{3,2}$ has relative volume at most $99.2\%$ inside the simplex $\Delta_7$~\cite{montufar2015does}. Simulations using Theorem~\ref{prop} and Corollary \ref{cor:m33} estimate the true volume of both of these models to be $75.3\%$. 

We use Theorem \ref{prop} to prove a conjecture from \cite[Section 3.5.1]{montufar2015does}: 

\begin{corollary} \label{conj:MM}
No distribution in $\RBM_{3,2}$ has four modes. 
 \end{corollary}
For a discrete distribution, a \emph{mode} is a state with larger probability than any of its Hamming neighbour states. 
Corollary \ref{conj:MM} is stated as a conjecture $\RBM_{3,2} \cap \mathcal{G}_3 = \emptyset$ in \cite{montufar2015does}, where $\mathcal{G}_3$ denotes distributions on $\{ 0, 1\}^3$ with four modes (the maximum possible number). 
Note that the models $\mathcal{M}_{3,4}$ and $\RBM_{3,3}$ fill the interior of the simplex $\Delta_7$ \cite{montufar2013mixture, montufar2016hierarchical}. 
Corollary~\ref{conj:MM} also follows from Theorem~\ref{conj:eq}, since no $p \in \Mcal_{3,3}$ has four modes~\cite[Proposition~3.10]{montufar2015does}. 

\medskip

This remainder of the paper is organized as follows.
We derive the implicit description of $\RBM_{3,2}$ in Section \ref{rbm32}. 
We obtain the equality of $\RBM_{3,2}$ and $\mathcal{M}_{3,3}$ in Section~\ref{m33}. 
We connect this description to triangulations of the three-cube in Section \ref{triang}, where we also prove Corollary \ref{conj:MM}. 
We describe the boundary of the model $\mathcal{M} 
= \Mcal_{3,3} = \overline{\RBM_{3,2}}$ in Section \ref{5}, and we study the maximum likelihood problem for the model in Section \ref{6}. 
We explain how to construct a three-dimensional visualization of the model in Section \ref{7}. 
Finally, in Section~\ref{section:conclusions} we study extensions to $n$ binary random variables. 

\section{The semi-algebraic description of $\boldsymbol{\RBM_{3,2}}$} \label{rbm32}

We first recall the semi-algebraic description of the non-negative rank at most two model $\mathcal{M}_{3,2}$ given in \cite{ALLMAN201537}.
The model is described in $\Delta_7$ by the union of four basic semi-algebraic sets. 
On the interior of the simplex, one of the sets is given by the inequalities
\begin{equation} \label{ineq} \begin{matrix} p_{000} p_{011} \geq p_{010} p_{001} , & p_{000} p_{101} \geq p_{100} p_{001}, & p_{000} p_{110} \geq p_{100} p_{010} , \\
p_{100} p_{111} \geq p_{110} p_{101} , & p_{010} p_{111} \geq p_{110} p_{011} , &  p_{001} p_{111} \geq p_{101} p_{011}.
\end{matrix} \end{equation}
The other three sets are obtained 
by reversing the signs of the inequalities in {\em two out of the three} columns of \eqref{ineq}. For example: 
\begin{equation} \label{ineq2} \begin{matrix} p_{000} p_{011} \geq p_{010} p_{001} , & p_{000} p_{101} \leq p_{100} p_{001}, & p_{000} p_{110} \leq p_{100} p_{010} , \\
p_{100} p_{111} \geq p_{110} p_{101} , & p_{010} p_{111} \leq p_{110} p_{011} , & p_{001} p_{111} \leq p_{101} p_{011} .
\end{matrix} \end{equation}

One way to get a distribution in $\RBM_{3,2}$ is to take the Hadamard product of a distribution satisfying \eqref{ineq} with one satisfying \eqref{ineq2}. 
We find the semi-algebraic description for all distributions expressible as such a Hadamard product. It is defined by the polynomial inequalities in \eqref{ineq3}. From this, swapping indices gives the full semi-algebraic description of the restricted Boltzmann machine $\RBM_{3,2}$ on the interior of the simplex. Note that the independence model $\mathcal{M}_{3,1}$ is obtained on the interior of $\Delta_7$ by setting the inequalities in~\eqref{ineq} or~\eqref{ineq2} to equalities. 

\subsection{On the interior of the simplex} 

The binomial inequalities above translate to linear inequalities in the log-probabilities. For a strictly positive distribution $p = (p_{ijk})$, we take the log distribution $l_{ijk} = {\rm log}(p_{ijk})$. 
Taking the logarithm of the inequalities in \eqref{ineq} gives the polyhedron
\begin{equation*} X = \left\{ \begin{matrix}  l_{000} + l_{011} - l_{001} - l_{010} \geq 0, &  l_{100} + l_{111} - l_{101} - l_{110} \geq 0 \\ 
l_{000} + l_{101} - l_{001} - l_{100} \geq 0, &  l_{010} + l_{111} - l_{011} - l_{110} \geq 0 \\ 
l_{000} + l_{110} - l_{010} - l_{100} \geq 0, &  l_{001} + l_{111} - l_{011} - l_{101} \geq 0 \\ \end{matrix} \right\} .
\end{equation*}
Similarly, we define $Y$ to be the log-probabilities satisfying the logarithms of \eqref{ineq2}, 
\begin{equation*} Y = \left\{ \begin{matrix}  l_{000} + l_{011} - l_{001} - l_{010} \geq 0, &  l_{100} + l_{111} - l_{101} - l_{110} \geq 0 \\ 
l_{000} + l_{101} - l_{001} - l_{100} \leq 0, &  l_{010} + l_{111} - l_{011} - l_{110} \leq 0 \\ 
l_{000} + l_{110} - l_{010} - l_{100} \leq 0, &  l_{001} + l_{111} - l_{011} - l_{101} \leq 0 \\ \end{matrix} \right\} .
\end{equation*}
Taking the Hadamard product in probability space is the same as taking the
sum in log-probability space. Therefore, showing Theorem \ref{prop} is equivalent to proving that the Minkowski sum $X + Y = \{ x + y : x \in X, y \in Y \}$ is
\begin{equation*} W = \{ l_{000} + l_{011} - l_{001} - l_{010} \geq 0 ,  \quad l_{100} + l_{111} - l_{101} - l_{110}  \geq 0 \} .\end{equation*}

The two polyhedra $X$ and $Y$ are eight-dimensional in $\RR^8$.
The lineality spaces of a polyhedron is the space obtained by setting all the inequalities in their descriptions to equalities. For both $X$ and $Y$, the lineality space is the set of tensors $(l_{ijk})$ for which the tensor $({\rm exp}(l_{ijk}))$ is rank one. It is spanned by the rows of the matrix
\[
\begin{blockarray}{cccccccc}
l_{000} & l_{100} & l_{010} & l_{001} & l_{110} & l_{101} & l_{011} & l_{111}  \\
\begin{block}{(cccccccc)}
  1 & 1 & 1 & 1 & 1 & 1 & 1 & 1 \\
    0 & 1 & 0 & 0 & 1 & 1 & 0 & 1 \\
    0 & 0 & 1 & 0 & 1 & 0 & 1 & 1 \\
    0 & 0 & 0 & 1 & 0 & 1 & 1 & 1 \\
\end{block}
\end{blockarray} \text{ }.
 \]
The polyhedron $W$ is also eight-dimensional. It has a six-dimensional lineality space that is spanned degenerately by the rows of the matrix
\begin{equation} \label{spanning}
\begin{blockarray}{cccccccc}
l_{000} & l_{100} & l_{010} & l_{001} & l_{110} & l_{101} & l_{011} & l_{111}  \\
\begin{block}{(cccccccc)}
  1 & 0 & 1 & 0 & 0 & 0 & 0 & 0 \\
  1 & 0 & 0 & 1 & 0 & 0 & 0 & 0 \\
    0 & 0 & 0 & 0 & 1 & 0 & 0 & 1 \\
  0 & 0 & 0 & 0 & 0 & 1 & 0 & 1 \\
  0 & 0 & 1 & 0 & 0 & 0 & 1 & 0 \\
  0 & 0 & 0 & 1 & 0 & 0 & 1 & 0 \\
  0 & 1 & 0 & 0 & 1 & 0 & 0 & 0 \\
  0 & 1 & 0 & 0 & 0 & 1 & 0 & 0 \\
\end{block}
\end{blockarray} \text{ }.
\end{equation}
Using the software \verb|polymake| \cite{GJ}, we can find a description for the quotient of $X$ or $Y$ by its lineality space. They are both triangular bipyramids. 

\begin{proof}[Proof of Theorem \ref{prop}]

We aim to show that $W = X + Y$. We begin with the containment $X + Y \subseteq W$. Summing the first equations in $X$ and $Y$ yields 
$$  x_{000} + y_{000} + x_{011} + y_{011} - x_{001} - y_{001} - x_{010} - y_{010} \geq 0 $$
while summing the second equations from $X$ and $Y$ gives
$$ x_{100} + y_{100} + x_{111} + y_{111} - x_{101} - y_{101} - x_{110} - y_{110} \geq 0 .$$
Translating back to the $l$-coordinates, we get 
$ l_{000} + l_{011} - l_{001} - l_{010} \geq 0$ and $l_{100} + l_{111} - l_{101} - l_{110}  \geq 0 $. 
Hence $X + Y \subseteq W$. 

For the reverse containment $W \subseteq X + Y$ we require a spanning set for $W$ in which every basis vector lies either in $X$ or in $Y$. The first four rows of the lineality space of $W$ in \eqref{spanning} lie in $X$, while the last four rows lie in $Y$. Hence any non-negative combination of the lineality space lies in $W$. To extend to negative linear combinations we multiply the spanning set by $-1$. 
The first four rows of the negation of \eqref{spanning} lie in $Y$, and the last four are in $X$. 

It remains to find a basis for the two-dimensional polytope obtained by taking the quotient of $W$ by its lineality space. The quotient is spanned by non-negative combinations of any two linearly independent vectors in $W$ not in its lineality space. For example $l_{000} \in X$ and $l_{100} \in Y$. All non-negative combinations of these lie in $X + Y$. This concludes the proof. \end{proof}

\subsection{On the boundary of the simplex} 
\label{hi}

Theorem \ref{prop} gives a semi-algebraic description for the restricted Boltzmann machine $\RBM_{3,2}$ on the interior of the simplex $\Delta_7$. 
However, for $p$ in the boundary of the simplex $\partial \Delta_7$, the inequalities in Theorem~\ref{prop} are not sufficient for membership in $\RBM_{3,2}$. 

\begin{proposition} \label{prop:bd}
The intersection $\RBM_{3,2} \cap \partial \Delta_7$ is given by distributions which satisfy 
  \begin{equation}\label{star}
  \begin{matrix}
\text{If the probability of a state vanishes, so does the } \\
\text{ probability of one of its Hamming neighbour states.} 
\end{matrix}
  \end{equation}
\end{proposition}

\begin{proof}
First we show that $p \in \RBM_{3,2} \cap \partial \Delta_7$ satisfies condition \eqref{star}. Since $p$ lies on the boundary of $\Delta_7$, one of its entries vanishes. 
Assume without loss of generality $p_{000} = 0$. Then condition \eqref{star} means that
$p_{100} p_{010} p_{001} = 0$.
Since $p \in \RBM_{3,2}$, it is the product of two distributions in $\mathcal{M}_{3,2}$. 
That is, 
$$p_{ijk} = (q_{ijk} + r_{ijk}) ( s_{ijk} + t_{ijk} ),$$
where $q, r, s, t$ are rank one non-negative $2 \times 2 \times 2$ tensors. Up to swapping factors the $(0,0,0)$ entry of the tensor $q + r$ must vanish. Hence $q_{000} = r_{000} = 0$. Since $q$ and $r$ are rank one, they must vanish on a slice. Both $q$ and $r$ vanish in at least one of the locations $(0,0,1)$, $(0,1,0)$ and $(1,0,0)$, hence so does $p$.

For the converse, we consider some $p \in \partial \Delta_7$ satisfying \eqref{star} and we aim to show that $p \in \RBM_{3,2}$. As before, we can assume $p_{000} = 0$. Condition \eqref{star} implies that one of $p_{001}, p_{010}, p_{100}$ must also vanish. We reorder indices such that $p_{010}$ vanishes. The distribution admits the Hadamard factorization
$$ p = \left[
\begin{array}{cc||cc}
0 & 0 & p_{001} & p_{011} \\ p_{100} & p_{110} & p_{101} & p_{111}
\end{array}
\right] = \left[
\begin{array}{cc||cc}
0 & 0 & p_{001} & p_{011} \\ p_{101} & p_{111} & p_{101} & p_{111}
\end{array}
\right] * \left[
\begin{array}{cc||cc}
0 & 0 & 1 & 1 \\ \frac{p_{100}}{p_{101}} & \frac{p_{110}}{p_{111}} & 1 & 1 \end{array}
\right].$$
If $p_{101},p_{111} \neq 0$, both factors are non-negative rank two and the distribution lies in $\RBM_{3,2}$. If $p_{101} = 0$, then $p_{111}p_{100}p_{001} = 0$ and if $p_{111} = 0$ then $p_{110}p_{101}p_{011} = 0$. In both of these cases the distribution consists of two pairs of non-zero adjacent entries, hence lies in $\mathcal{M}_{3,2}$, which is a subset of $\RBM_{3,2}$. Hence in all cases the distribution lies in $\RBM_{3,2}$.
\end{proof}

Condition \eqref{star} is stricter than the restriction of the inequalities in Theorem \ref{prop} to the boundary of the simplex. 
The model $\RBM_{3,2}$ is a semi-algebraic subset of the simplex that is not closed. We give an example of a distribution that lies in the closure of the model, but not in the model.

\begin{example} \label{counter}
Consider the distribution
$$ p_{ijk} = \begin{cases} \frac13, &  (i,j,k) = (0,0,1), (0,1,0), (1,0,0) \\
0, & \text{ otherwise} .\end{cases}$$
Observe that $p \in \mathcal{M}_{3,3}$, since $p = \frac13 (e_0 \otimes e_0 \otimes e_1 + e_0 \otimes e_1 \otimes e_0 + e_1 \otimes e_0 \otimes e_0 )$ has non-negative rank three and entries summing to one. Since $p$ does not satisfy the conditions in Proposition~\ref{prop:bd}, $p \notin \RBM_{3,2}$. We give a sequence of distributions $(p_n) \subset \RBM_{3,2}$, such that $p_n \to p$. Consider 
$$ p_n 
\propto
\left[
\begin{array}{cc||cc}
\epsilon & 1 & 1 & \epsilon \\ 1 & \epsilon & \epsilon & \epsilon^4
\end{array}
\right] ,$$
where $||$ divides the two slices of the tensor, and $\epsilon = \frac1n$. 
As $n \to \infty$, 
$p_n \to p$. 
The scaling factor can be subsumed to either factor in the following decomposition.  
$$ \hspace{-48ex} p_n 
\propto
 \left[ \begin{array}{cc||cc}
\epsilon & 1 & \epsilon^2 & \epsilon \\ 1 & \epsilon & \epsilon & \epsilon^2
\end{array} \right] *  \left[ \begin{array}{cc||cc}
1 & 1 & \epsilon^{-2} & 1 \\ 1 & 1 & 1 & \epsilon^2
\end{array} \right] $$
$$ \hspace{2.6ex} = \left( \begin{bmatrix} \epsilon \\ 1 \end{bmatrix} \otimes \begin{bmatrix} 1 \\ 0 \end{bmatrix} \otimes \begin{bmatrix} 1 \\ \epsilon \end{bmatrix} + \begin{bmatrix} 1 \\ \epsilon \end{bmatrix} \otimes \begin{bmatrix} 0 \\ 1 \end{bmatrix} \otimes \begin{bmatrix} 1 \\ \epsilon \end{bmatrix} \right) * \left( \begin{bmatrix} 1 \\ 1 \end{bmatrix} \otimes \begin{bmatrix} 1 \\ 1 \end{bmatrix} \otimes \begin{bmatrix} 1 \\ 0 \end{bmatrix} + \begin{bmatrix} \epsilon^{-1} \\  \epsilon \end{bmatrix} \otimes \begin{bmatrix}  \epsilon^{-1} \\  \epsilon \end{bmatrix} \otimes \begin{bmatrix} 0 \\ 1 \end{bmatrix} \right) $$
This decomposition shows that $p_n \in \RBM_{3,2}$ for each $n$. Hence $\RBM_{3,2}$ is not closed. 
\end{example}

In the example above, the entries of one of the tensors in the decomposition are unbounded as $n \to \infty$. They are multiplied by very small entries in the other term so that the limiting tensor $p$ is bounded. Such situations can be avoided on the interior of the simplex, where the model $\RBM_{3,2}$ is closed (see Lemma \ref{lemma}). 

\begin{proposition} \label{BW}
The model $\mathcal{M}_{n,k}$ is closed for all $n$ and $k$.
\end{proposition}

\begin{proof}
Consider a convergent sequence of tensors $p_n \to p$, where each $p_n \in \mathcal{M}_{n,k}$. We show that the limiting tensor $p$ also lies in $\mathcal{M}_{n,k}$. Each $p_n$ can be written as the sum of $k$ non-negative rank one tensors $p_n = a_n + b_n + \cdots + c_n$. Since the entries of $p_n$ are bounded above by 1, and the entries of $a_n, b_n, \ldots, c_n$ are non-negative, the entries of $a_n, b_n , \ldots, c_n$ are also bounded above by 1. By the Bolzano Weierstrass Theorem, there exists a subsequence of the $a_n$, call it $a_{n_j}$, that  converges. Its limit, $a$, is a non-negative rank one tensor. Taking $p_{n_j} \to p$ as our new convergent sequence, we repeat the argument to find a convergent subsequence of the $b_{n_j}$ which converges to a non-negative rank one tensor $b$. Repeating $k$ times we obtain a subsequence of the $p_n$ whose limit is $a + b + \cdots + c$. Hence $p = a + b + \cdots + c \in \mathcal{M}_{n,k}$.
\end{proof}

\section{Equality of $\boldsymbol{\RBM_{3,2}}$ and $\boldsymbol{\mathcal{M}_{3,3}}$} \label{m33}

We prove Theorem \ref{conj:eq} by proving the two directions of the containment in two lemmas. 
The second sentence of the theorem (equality on the interior of the simplex) follows from the first (equality of the model closures) by the fact that $\RBM_{3,2}$ is closed on the interior of the simplex (see Lemma \ref{lemma}). 

\begin{lemma} \label{ab}
We have the containment of statistical models $\RBM_{3,2} \subseteq \mathcal{M}_{3,3}$.
\end{lemma}

\begin{proof}
Consider a distribution $p \in \RBM_{3,2}$. If $p \in \partial \Delta_7$ then it satisfies \eqref{star} and we can assume without loss of generality $p_{000} = p_{001} = 0$. Then
$$ p = \left[
\begin{array}{cc||cc}
0 & 0 & 0 & 0 \\ p_{100} & 0 & p_{101} & 0
\end{array}
\right] + \left[
\begin{array}{cc||cc}
0 & p_{010} & 0 & p_{011} \\ 0 & 0 & 0 & 0
\end{array}
\right] + \left[
\begin{array}{cc||cc}
0 & 0 & 0 & 0 \\ 0 & p_{110} & 0 & p_{111} \end{array}
\right]$$
is an expression for $p$ as the sum of three non-negative rank one terms, hence $p \in \mathcal{M}_{3,3}$.

It remains to consider distributions $p$ with no
entries vanishing. We name the six determinants by $d_{i,j}$ where $i \in \{1,2,3\}$ denotes which index is fixed in the determinant, and $j \in \{0,1\}$ gives the value of the fixed index:
\begin{equation} \label{dets} 
\begin{matrix} 
d_{1,0} = p_{000} p_{011} - p_{001} p_{010}, & d_{1,1} = p_{100} p_{111} - p_{101} p_{110}, \\ 
d_{2,0} = p_{000} p_{101} - p_{001} p_{100}, & d_{2,1} = p_{010} p_{111} - p_{011} p_{110}, \\ 
d_{3,0} = p_{000} p_{110} - p_{010} p_{100}, & d_{3,1} = p_{001} p_{111} - p_{011} p_{101}. 
\end{matrix} \end{equation}
As we will see in Section \ref{triang} and Figure \ref{fig1}b, 
we can relabel indices such that  determinants $d_{2,1}$ and $d_{1,1}$ have opposite signs. We can write $p$ as
$$\label{p} p = \left[
\begin{array}{cc||cc}
p_{000} & 0 & p_{001} & 0 \\ 0 & 0 & 0 & 0
\end{array}
\right] + \left[
\begin{array}{cc||cc}
0 & 0 & 0 & 0 \\ p_{100} & x & p_{101} & \frac{p_{101}}{p_{100}} x
\end{array}
\right] + \left[
\begin{array}{cc||cc}
0 & p_{010} & 0 & p_{011} \\ 0 & y & 0 & \frac{p_{011}}{p_{010}} y \end{array}
\right], $$
where
$ x = \frac{ p_{100} p_{111} \cdot d_{2,1} }{ p_{101} d_{2,1} - p_{011} d_{1,1} }$ and $y = \frac{p_{010} p_{111} \cdot d_{1,1} }{ p_{011} d_{1,1} - p_{101} d_{2,1} } $.
Since the signs of $d_{2,1}$ and $d_{1,1}$ are different this expression for $p$ is non-negative rank three, hence $p \in \mathcal{M}_{3,3}$. The denominator of $x$ and $y$ is non-zero, provided that $d_{2,1}$ or $d_{1,1}$ is non-zero. If some determinant vanishes, a non-negative rank three decomposition is obtained from the rank one tensor of that face plus the non-negative rank two representation of the opposite face.

Note that $x$ and $y$ are not both non-negative for $p \notin \RBM_{3,2}$ by Figure~\ref{2}: there is no way to rotate or reflect the cube such that determinants $d_{2,1}$ and $d_{2,2}$ have opposite sign.
\end{proof}

\begin{lemma} \label{ba}
We have the containment of statistical models $\mathcal{M}_{3,3} \subseteq \overline{\RBM_{3,2}}$.
\end{lemma}

\begin{proof} 
Consider a distribution $p + q \in \mathcal{M}_{3,3}$ where $p$ is non-negative rank two, $q$ is non-negative rank one, and no entries of $p$ or $q$ vanish. Up to swapping values $0$ and $1$ in one index, $p$ being non-negative rank two means it satisfies the six binomial inequalities in \eqref{ineq}. Equivalently, its determinants $d_{i,j}$ from \eqref{dets} have sign pattern $(+,+,+,+,+,+)$. We assume for contradiction that $p + q \notin \RBM_{3,2}$. This means $p + q$ has three ``$-$" in its sign pattern. After adding
tensor $q$, {\em three} determinants have swapped sign: $d_{1,h}$, $d_{2,h}$,
$d_{3,h}$ for $h = 0$ or $1$. 

Take non-negative vectors $a, b, c \in \mathbb{R}_{\geq 0}^2$ such that $q_{ijk} = a_i b_j c_k$. Assume 
determinant $d_{3,h}$ of $p + q$ is negative: $( p_{00h} + a_0 b_0 c_h) ( p_{11h} + a_1 b_1 c_h ) - ( p_{01h} + a_0 b_1 c_h )(p_{10h} + a_1 b_0 c_h ) < 0$.
Multiplying this expression out, and using $p_{00h} p_{11h} \geq p_{01h} p_{10h}$, gives
\begin{equation} \label{i} 
p_{00h} a_1 b_1 + p_{11h} a_0 b_0 < p_{10h} a_0 b_1 + p_{01h} a_1 b_0 .
\end{equation}
Hence either $p_{00h} b_1 < p_{01h} b_0$ or $p_{11h} b_0 < p_{10h} b_1$ must hold, and likewise either $p_{00h} a_1 < p_{10h} a_0$ or $p_{11h} a_0 < p_{01h} a_1$ must hold. Furthermore, rearranging \eqref{i} yields
$$ \frac{1}{p_{00h}} ( p_{00h} a_1 - p_{10h} a_0 ) ( p_{00h} b_1 - p_{01h} b_0 ) + \left( p_{11h} - \frac{p_{10h} p_{01h}}{p_{00h}} \right) a_0 b_0 < 0.$$
Since the last term is non-negative, this implies that
$\frac{1}{p_{00h}} ( p_{00h} a_1 - p_{10h} a_0 ) ( p_{00h} b_1 - p_{01h} b_0 ) < 0$,
hence exactly one of $p_{00h} a_1 < p_{10h} a_0$ and $p_{00h} b_1 < p_{01h} b_0$ holds. Similarly, \eqref{i} yields 
$$ \frac{1}{p_{11h}} ( p_{11h} a_0 - p_{01h} a_1 ) ( p_{11h} b_0 - p_{10h} b_1 ) + \left( p_{00h} - \frac{p_{01h} p_{10h}}{p_{11h}} \right) a_1 b_1 < 0 ,$$ 
implying exactly one of $p_{11h} a_0 < p_{01h} a_1$ and $p_{11h} b_0 < p_{10h} b_1$ holds. Repeating the above for determinants $d_{2,h}$ and $d_{1,h}$ gives the following $2^3 = 8$ options:
$$ \begin{array}{lllll} I_{ab}^{(1)} = \{ p_{00h} b_1 < p_{01h} b_0, & p_{11h} a_0 < p_{01h} a_1 \}, & I_{ab}^{(2)} = \{ p_{11h} b_0 < p_{10h} b_1, & p_{00h} a_1 < p_{10h} a_0 \}, \\
 I_{ac}^{(1)} = \{ p_{0h0} a_1 < p_{1h0} a_0 , & p_{1h1} c_0 < p_{1h0} c_1 \}, &  I_{ac}^{(2)} = \{ p_{1h1} a_0 < p_{0h1} a_1 , & p_{0h0} c_1 < p_{0h1} c_0 \} , \\
I_{bc}^{(1)} =  \{ p_{h00} c_1 < p_{h01} c_0 , & p_{h11} b_0 < p_{h01} b_1 \}, &  I_{bc}^{(2)} =  \{ p_{h11} c_0 < p_{h10} c_1 , & p_{h00} b_1 < p_{h10} b_0 \} .\end{array}$$
If either inequality from $I_{ab}^{(1)}$ is satisfied, the inequalities of $I_{ab}^{(2)}$ cannot be satisfied, and likewise for $I_{ac}$ and $I_{bc}$.
To conclude the proof, we derive a contradiction from these options.

Let $h = 0$. 
Assume the inequalities in $I_{ab}^{(1)}$ hold. Then one of the inequalities from $I_{bc}^{(2)}$ is satisfied, hence $I_{bc}^{(1)}$ cannot hold. If $I_{ac}^{(1)}$ also holds, combining $p_{110} a_0 < p_{010} a_1$ from $I_{ab}^{(1)}$ with $p_{000} a_1 < p_{100} a_0$ from $I_{ac}^{(1)}$ gives $p_{110} p_{000} < p_{010} p_{100}$, contradicting the hypothesis that $p$ satisfies the inequalities in \eqref{ineq}. If $I_{ac}^{(2)}$ holds, combining inequalities involving $c$ gives $p_{000} p_{011} < p_{001} p_{010}$, also a contradiction. Likewise, if $I_{ab}^{(2)}$ holds then $I_{ac}^{(1)}$ must hold. If $I_{bc}^{(1)}$ also holds, combining the inequalities involving $c$ implies $p_{101} p_{000} < p_{100} p_{001}$, a contradiction. If $I_{bc}^{(2)}$ holds, combining inequalities involving $b$ gives $p_{110} p_{000} < p_{100} p_{010}$, also a contradiction. The case $h = 1$ follows by analogous reasoning.

This shows that an open dense subset of $\mathcal{M}_{3,3}$ is contained in $\RBM_{3,2}$. It remains to consider when $p$ or $q$ has some vanishing entry. Such cases are in the closure of the above, hence they lie in the closure of $\RBM_{3,2}$. 
\end{proof}

\section{Connection to triangulations of the three-cube} \label{triang}

Let $\Mcal$ be the statistical model $\mathcal{M}_{3,3} = \overline{\RBM_{3,2}}$.  
We characterize $\Mcal$ on the interior of $\Delta_7$ in terms of triangulations. This allows us to prove Corollary \ref{conj:MM}. We describe below how to triangulate the three-cube using a positive distribution $p \in \Delta_7$. Membership in $\Mcal$ is determined by how this triangulation restricts to the faces of the cube. 

\begin{proposition}\label{figone}
The model $\Mcal$ contains distributions whose triangulations restrict to the faces of the cube as in Figure \ref{fig1}. 
Distributions whose triangulations restrict as in Figure \ref{2} lie outside of $\Mcal$. 
Triangulations in Figure \ref{4} are special cases of those in Figure \ref{fig1} and come from distributions in $\mathcal{M}_{3,2}$.
\end{proposition}

A diagonal lines on a face of a cube in Figures \ref{4}, \ref{fig1}, and \ref{2} indicates the direction that the face is triangulated. 
The empty faces in Figure \ref{fig1} can be triangulated in either of the two possible directions. Relabeling indices does not change membership in our statistical models. 
It corresponds to rotating or reflecting the cubes.

\begin{figure}
	\begin{minipage}{0.25\textwidth}
		\centering
		\includegraphics[width=3cm]{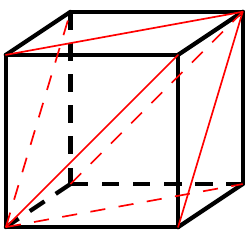}
		\caption{Distributions in $\mathcal{M}_{3,2}$ give (rotations of) this triangulation.}
		\label{4}
	\end{minipage}
	\begin{minipage}{0.03\textwidth}
		\qquad
	\end{minipage}
	\begin{minipage}{0.4\textwidth}
		\centering
		\includegraphics[width=3cm]{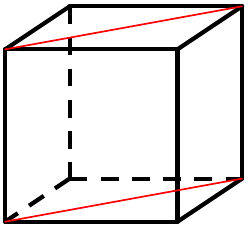}a
		\caption{Two characterizations of the triangulations from $\Mcal$ (up to rotation). Empty faces can be triangulated in either direction.}
		\includegraphics[width=3cm]{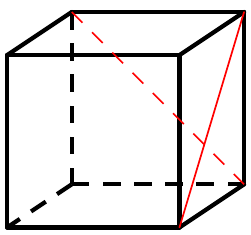}b
		\label{fig1}
	\end{minipage}
	\begin{minipage}{0.03\textwidth}
		\qquad
	\end{minipage}
	\begin{minipage}{0.25\textwidth}
		\centering
		\includegraphics[width=3cm]{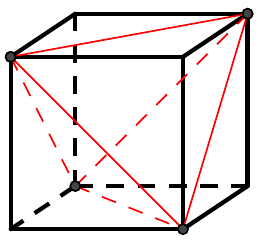}
		\caption{Distributions not in $\Mcal$ give this triangulation  (up to rotation) .}
		\label{2}
	\end{minipage}
\end{figure}

Consider a generic, strictly positive distribution $p \in \Delta_7$. Its tensor $(l_{ijk}) = {\rm log}(p_{ijk})$ of log-probabilities induces a triangulation of the three-cube. For two observed variables, the set-up is shown in \cite[Figure 1]{BPS}. In three dimensions, we do the higher-dimensional analogue: we assign height $l_{ijk}$ to the vertex of the three-cube with coordinates $(i,j,k)$. We take the convex hull of the heights in four-dimensional space.  Then we project the convex hull back to the three-dimensional cube. The facets in the convex hull project to tetrahedra in the cube that combine to make a triangulation. The three-cube has 74 possible triangulations which fall into six triangulation types, see  \cite[Figure 1]{HSYY}. In \cite{BPS} the authors study these triangulations in the context of epistasis in evolutionary biology. 

\begin{proof}[Proof of Proposition \ref{figone}]
 There are 20 linear expressions in the coordinates $l_{ijk}$ whose signs determine the triangulation, see \cite[page 1325]{BPS}. Six of these equations determine how the triangulation restricts to the faces of the cube. These are the logarithms of the binomial equations that define $\Mcal$. Hence we can see whether ${\rm exp}(l_{ijk})$ lies in $\Mcal$ by looking at how the triangulation induced by $(l_{ijk})$ restricts to the faces of the cube. The equations that define $\Mcal_{3,2}$ and $\Mcal_{3,1}$ are also of this form. 

In the language of triangulations, being in $\Mcal$ means the triangulation {\em slices at least one pair of opposite faces in the same direction}, as in Figure \ref{fig1}a. The condition for being in $\mathcal{M}_{3,2}$ is that every pair of opposite faces is sliced in the same direction, with sign compatibility as in Figure \ref{4}. Triangulations of distributions not in $\RBM_{3,2}$ {\em slice every pair of opposite faces in opposing directions}, as in Figure~\ref{2}. An alternate characterization of such triangulations is that every pair of adjacent faces is sliced in a continuous way. If conversely a pair of adjacent faces is sliced in a discontinuous way, as in Figure \ref{fig1}b, the distribution lies in $\Mcal$.
\end{proof}

We can re-phrase Proposition \ref{figone} in terms of the numbering of the triangulation types from \cite[Page 1657]{HSYY}. The model $\RBM_{3,2}$ only contains distributions with triangulation types 3, 4, 5 and 6. Triangulation types 1 and 2 come from distributions that lie outside of the model. Triangulation type 6 is from distributions in $\mathcal{M}_{3,2}$. Note that, in Figure \ref{fig1}, if at least one of the two other pairs of opposite faces are sliced in different directions we get a triangulation of type 3 or 5. If both other pairs are sliced in the same direction, but not with the right sign-compatibility for $\mathcal{M}_{3,2}$ membership, we have type 4.

\begin{proof}[Proof of Corollary \ref{conj:MM}]
	The idea of the proof is to show that distributions with four modes restrict to the faces of the cube as shown in Figure \ref{2}. Assume we have a distribution with four modes. Without loss of generality, the four numbers $l_{000}$, $l_{011}$, $l_{101}$, and $l_{110}$ exceed the values of their neighbours. Consider a face of the cube, for example the face $\langle l_{000}, l_{001}, l_{010}, l_{011} \rangle$. Since $l_{000} \geq l_{001}$ and $l_{011} \geq l_{010}$, we have
	$$ l_{000} + l_{011} - l_{010} - l_{001} \geq 0 ,$$
	which determines how the triangulation of $(l_{ijk})$ restricts to the face. Repeating for the other faces gives the triangulation of the faces shown in Figure \ref{2}. 
	
	Distributions on $\partial \Delta_7 \cap \RBM_{3,2}$ have at least two adjacent entries vanishing, by \eqref{star}. This excludes the possibility of having four modes.
\end{proof}

\section{The boundary of the model} \label{5}

We saw that the statistical model $\Mcal = \Mcal_{3,3} = \overline{\RBM_{3,2}}$ is defined by the binomial inequalities in Theorem \ref{prop}. Setting the inequalities in Theorem \ref{prop} to equalities gives the Zariski closure of the boundary of the model.

\begin{proposition} \label{slice} Distributions on the boundary of $\Mcal$ are given by $2 \times 2 \times 2$ tensors with a $2 \times 2$ slice of rank $\leq 1$.
\end{proposition}

That is, the Zariski closure of the boundary of the model is a union of hypersurfaces $\{ d_{i,j} = 0 \}$, for $1 \leq i \leq 3$, $0 \leq j \leq 1$. This is also the Zariski closure of the boundary of the model $\mathcal{M}_{3,2}$ from \cite{ALLMAN201537}. Proposition \ref{slice} says the boundary of $\Mcal$ consists of mixtures of three product distributions with disjoint supports in $\{0,1\}^3$. Mixtures of products with disjoint supports were used in~\cite{NIPS2011_4380} to study the representational power of RBMs. 

The following is a converse result. It implies that $\RBM_{3,2}$ is closed on the interior of the simplex. Furthermore, within the simplex of probability distributions, the Zariski closure of the boundary is contained in the closure of the model. This result (which fails for $\Mcal_{3,2}$) is useful in Section \ref{6} when we study maximum likelihood estimation.

\begin{lemma} \label{lemma}
Every distribution of three binary random variables with a rank one $2 \times 2$ slice, and strictly positive entries, lies in the models $\RBM_{3,2}$ and $\mathcal{M}_{3,3}$. 
\end{lemma}

\begin{proof}
As in the proof of Lemma \ref{ab}, if the determinant of a distribution $p$ vanishes, a non-negative rank three decomposition is obtained from the rank one tensor of that slice plus the non-negative rank two representation of the opposite slice. This proves the result for $\Mcal_{3,3}$.

It remains to build a decomposition of $p$ as $(q + r)(s + t)$ where $q, r, s, t$ are rank one non-negative $2 \times 2 \times 2$ tensors. Assume without loss of generality that $d_{3,1} = 0$. Let $q$ be the rank one tensor
with slices $q_{**1}$ and $p_{**1}$ equal, where $q_{**0}$ is set to be the smallest scalar multiple of $p_{**1}$ that zeros out an entry of $p_{**0}$. Then $p - q$ consists of at most three non-zero entries. Let $r$ be the tensor which satisfies $r_{ijk} = p_{ijk} - q_{ijk}$ for two of the three entries at which $p \neq q$. Since these two entries can be chosen to be Hamming neighbours, $r$ is rank one. And since $p-q$ is non-negative, $r$ is non-negative. There remains at most one entry where equality $p = q + r$ does not hold: let $i,j,k$ be such that $p_{ijk} > q_{ijk} + r_{ijk}$. Let $s$ be the all ones tensor, and let $t$ be the tensor with just one non-zero entry, $t_{ijk} = \frac{p_{ijk}}{q_{ijk} + r_{ijk}} - 1$. Then $t$ is also non-negative and rank one, and $p = (q + r)(s + t)$ as required.
\end{proof}

In the log-probability coordinates, the boundary of $\Mcal$ is the union of hyperplanes:
\begin{equation}\label{Ws} \begin{matrix} \Lcal_{1,0} = \{ l_{000} + l_{011} - l_{001} - l_{010} = 0 \}, &  \Lcal_{1,1} = \{ l_{100} + l_{111} - l_{101} - l_{110} = 0 \}, \\ 
\Lcal_{2,0} = \{ l_{000} + l_{101} - l_{001} - l_{100} = 0 \}, &  \Lcal_{2,1} = \{ l_{010} + l_{111} - l_{011} - l_{110} = 0 \}, \\ 
\Lcal_{3,0} = \{ l_{000} + l_{110} - l_{010} - l_{100} = 0 \}, &  \Lcal_{3,1} = \{ l_{001} + l_{111} - l_{011}  - l_{101} = 0 \} .\\ \end{matrix}
\end{equation}

The intersection poset of a hyperplane arrangement is the set of all intersections of hyperplanes, ordered by reverse inclusion \cite{Stanley04anintroduction}.  In Figure~\ref{figure:intersectionposet} we give the intersection poset of the pieces of the boundary of $\mathcal{M}$. 
As an example of its non-generic structure, in Figure~\ref{figure:intersectionposet} we highlight three codimension three flats that are intersections of four hyperplanes.

We can study the combinatorics of the arrangement using its characteristic polynomial $\chi(t) = \sum_f \mu(f) t^{\dim(f)}$. The summation is taken over all flats in the arrangement, and $\mu$ is the M\"obius function (indicated in Figure~\ref{figure:intersectionposet} next to each node). Evaluating the characteristic polynomial at $t=-1$ gives the number of full dimensional regions of the ambient space defined by the arrangement (see~\cite{Stanley04anintroduction}) 
$$
|\chi(-1)| = 46. 
$$
For comparison, a generic four dimensional central arrangement of six hyperplanes defines $52$ regions. 
Ours is a central arrangement (the origin is in all hyperplanes) hence all $46$ regions are unbounded cones. Of the $46$ regions the model $\Mcal$ occupies $44$. The model $\Mcal_{3,2}$ occupies four of the regions. 

\begin{figure}
	\centering
	\scalebox{.85}{	
		\begin{tikzpicture}[vertex/.style={draw,circle, fill=black, inner sep =.075cm,outer sep=.1cm}]
		\foreach [count=\i] \coord in 
		{   (0,0), 
			(-5,2), (-3,2), (-1,2), (1,2), (3,2), (5,2), 
			(-7,4), (-6,4), (-5,4), (-4,4), (-3,4), (0,4), (-2,4), (-1,4), (1,4), (2,4), (3,4), (4,4), (5,4), (6,4), (7,4),
			(-5,6), (0,6), (-4,6), (-3,6), (-2,6), (-1,6), (5,6), (1,6), (2,6), (3,6), (4,6),
			(0,8)
		}{ \node[vertex] (p\i) at \coord {};} 
		
		\node[draw,circle, inner sep=.1cm,fill=black] at (-7,4) {};
		\node[draw,circle, inner sep=.1cm,fill=black] at (0,4) {};
		\node[draw,circle, inner sep=.1cm,fill=black] at (7,4) {};
		\node[draw,circle, inner sep=.1cm,fill=black] at (-5,6) {};
		\node[draw,circle, inner sep=.1cm,fill=black] at (0,6) {};
		\node[draw,circle, inner sep=.1cm,fill=black] at (5,6) {};
		\node[draw,circle, inner sep=.1cm,fill=black] at (0,8) {};
		
		\definecolor{bd1}{rgb}{0, 0, 0.5625};
		\definecolor{bd2}{rgb}{0, 0.1250, 1.0000}
		\definecolor{bd3}{rgb}{0, 0.800, 0.7500};
		\definecolor{bd4}{rgb}{0, 0.8000, 0.50};
		\definecolor{bd5}{rgb}{1.0000, 0.8500, 0};
		\definecolor{bd6}{rgb}{1.0000, 1, 0};
		\node[draw,circle, inner sep=.1cm,fill=bd1] at (-5,2) {};
		\node[draw,circle, inner sep=.1cm,fill=bd2] at (-3,2) {};
		\node[draw,circle, inner sep=.1cm,fill=bd3] at (-1,2) {};
		\node[draw,circle, inner sep=.1cm,fill=bd4] at (1,2) {};
		\node[draw,circle, inner sep=.1cm,fill=bd5] at (3,2) {};
		\node[draw,circle, inner sep=.1cm,fill=bd6] at (5,2) {};

		\foreach [count=\r] \row in 
		{
			{0,0,0,0,0,0,0,0,0,0,0,0,0,0,0,0,0,0,0,0,0,0,0,0,0,0,0,0,0,0,0,0,0}, 	
			{1,0,0,0,0,0,0,0,0,0,0,0,0,0,0,0,0,0,0,0,0,0,0,0,0,0,0,0,0,0,0,0,0},
			{1,0,0,0,0,0,0,0,0,0,0,0,0,0,0,0,0,0,0,0,0,0,0,0,0,0,0,0,0,0,0,0,0},
			{1,0,0,0,0,0,0,0,0,0,0,0,0,0,0,0,0,0,0,0,0,0,0,0,0,0,0,0,0,0,0,0,0},
			{1,0,0,0,0,0,0,0,0,0,0,0,0,0,0,0,0,0,0,0,0,0,0,0,0,0,0,0,0,0,0,0,0},
			{1,0,0,0,0,0,0,0,0,0,0,0,0,0,0,0,0,0,0,0,0,0,0,0,0,0,0,0,0,0,0,0,0},
			{1,0,0,0,0,0,0,0,0,0,0,0,0,0,0,0,0,0,0,0,0,0,0,0,0,0,0,0,0,0,0,0,0},
			{0,1,1,0,0,0,0,0,0,0,0,0,0,0,0,0,0,0,0,0,0,0,0,0,0,0,0,0,0,0,0,0,0},
			{0,1,0,1,0,0,0,0,0,0,0,0,0,0,0,0,0,0,0,0,0,0,0,0,0,0,0,0,0,0,0,0,0},
			{0,0,1,1,0,0,0,0,0,0,0,0,0,0,0,0,0,0,0,0,0,0,0,0,0,0,0,0,0,0,0,0,0},
			{0,1,0,0,1,0,0,0,0,0,0,0,0,0,0,0,0,0,0,0,0,0,0,0,0,0,0,0,0,0,0,0,0},
			{0,0,1,0,1,0,0,0,0,0,0,0,0,0,0,0,0,0,0,0,0,0,0,0,0,0,0,0,0,0,0,0,0},
			{0,0,0,1,1,0,0,0,0,0,0,0,0,0,0,0,0,0,0,0,0,0,0,0,0,0,0,0,0,0,0,0,0},
			{0,1,0,0,0,1,0,0,0,0,0,0,0,0,0,0,0,0,0,0,0,0,0,0,0,0,0,0,0,0,0,0,0},
			{0,0,1,0,0,1,0,0,0,0,0,0,0,0,0,0,0,0,0,0,0,0,0,0,0,0,0,0,0,0,0,0,0},
			{0,0,0,1,0,1,0,0,0,0,0,0,0,0,0,0,0,0,0,0,0,0,0,0,0,0,0,0,0,0,0,0,0},
			{0,0,0,0,1,1,0,0,0,0,0,0,0,0,0,0,0,0,0,0,0,0,0,0,0,0,0,0,0,0,0,0,0},
			{0,1,0,0,0,0,1,0,0,0,0,0,0,0,0,0,0,0,0,0,0,0,0,0,0,0,0,0,0,0,0,0,0},
			{0,0,1,0,0,0,1,0,0,0,0,0,0,0,0,0,0,0,0,0,0,0,0,0,0,0,0,0,0,0,0,0,0},
			{0,0,0,1,0,0,1,0,0,0,0,0,0,0,0,0,0,0,0,0,0,0,0,0,0,0,0,0,0,0,0,0,0},
			{0,0,0,0,1,0,1,0,0,0,0,0,0,0,0,0,0,0,0,0,0,0,0,0,0,0,0,0,0,0,0,0,0},
			{0,0,0,0,0,1,1,0,0,0,0,0,0,0,0,0,0,0,0,0,0,0,0,0,0,0,0,0,0,0,0,0,0},
			{0,0,0,0,0,0,0,1,1,1,1,1,1,0,0,0,0,0,0,0,0,0,0,0,0,0,0,0,0,0,0,0,0},
			{0,0,0,0,0,0,0,1,0,0,0,0,0,1,1,0,0,1,1,0,0,1,0,0,0,0,0,0,0,0,0,0,0},
			{0,0,0,0,0,0,0,0,1,0,0,0,0,1,0,1,0,0,0,0,0,0,0,0,0,0,0,0,0,0,0,0,0},
			{0,0,0,0,0,0,0,0,0,1,0,0,0,0,1,1,0,0,0,0,0,0,0,0,0,0,0,0,0,0,0,0,0},
			{0,0,0,0,0,0,0,0,0,0,1,0,0,1,0,0,1,0,0,0,0,0,0,0,0,0,0,0,0,0,0,0,0},
			{0,0,0,0,0,0,0,0,0,0,0,1,0,0,1,0,1,0,0,0,0,0,0,0,0,0,0,0,0,0,0,0,0},
			{0,0,0,0,0,0,0,0,0,0,0,0,1,0,0,1,1,0,0,1,1,1,0,0,0,0,0,0,0,0,0,0,0},
			{0,0,0,0,0,0,0,0,1,0,0,0,0,0,0,0,0,1,0,1,0,0,0,0,0,0,0,0,0,0,0,0,0},
			{0,0,0,0,0,0,0,0,0,1,0,0,0,0,0,0,0,0,1,1,0,0,0,0,0,0,0,0,0,0,0,0,0},
			{0,0,0,0,0,0,0,0,0,0,1,0,0,0,0,0,0,1,0,0,1,0,0,0,0,0,0,0,0,0,0,0,0},
			{0,0,0,0,0,0,0,0,0,0,0,1,0,0,0,0,0,0,1,0,1,0,0,0,0,0,0,0,0,0,0,0,0},
			{0,0,0,0,0,0,0,0,0,0,0,0,0,0,0,0,0,0,0,0,0,0,1,1,1,1,1,1,1,1,1,1,1},
		}{
		\foreach [count=\c] \cell in \row{
			\ifnum\cell=1%
			\draw (p\r) edge (p\c);
			\fi
		}
	}
	
	\foreach [count=\i] \coord in { 7}{ \node[circle, fill=white, inner sep=0] (q\i) at (-.4,8) {\textcolor{gray}{\small\coord}};} 	
	\foreach [count=\i] \coord in { -3,-1,-1,-1,-1,-3,-1,-1,-1,-1,-3}{ \node[circle, fill=white, inner sep=0] (q\i) at (\i-6.35,6) {\textcolor{gray}{\small\coord}};} 
	\foreach [count=\i] \coord in {1,1,1,1,1,1,1,1,1,1,1,1,1,1,1}{ \node[circle, fill=white, inner sep=0] (q\i) at (\i-8.3,4) {\textcolor{gray}{\small\coord}};} 
	\foreach [count=\i] \coord in {-1,-1,-1,-1,-1,-1}{ \node[circle, fill=white, inner sep=0] (q\i) at (2*\i-7.35,2) {\textcolor{gray}{\small\coord}};} 
	\foreach [count=\i] \coord in { 1}{ \node[circle, fill=white, inner sep=0] (q\i) at (-.325,0) {\textcolor{gray}{\small\coord}};} 	
	\end{tikzpicture}
}
\caption{Intersection poset of the boundary pieces of $\Mcal$. 
	The lowest node is the ambient space $\mathbb{R}^8$. 
	At the first level are the six boundary pieces. 
	At the second level are the $15$ pairwise intersections. The enlarged nodes are $\Lcal_{i,0}\cap\Lcal_{i,1}$. 
	The third level contains the $11$ distinct codimension three intersections. 
	The top intersection corresponds to the independence model. 
	The nodes are labeled with their M\"obius function value. }
\label{figure:intersectionposet}
\end{figure}
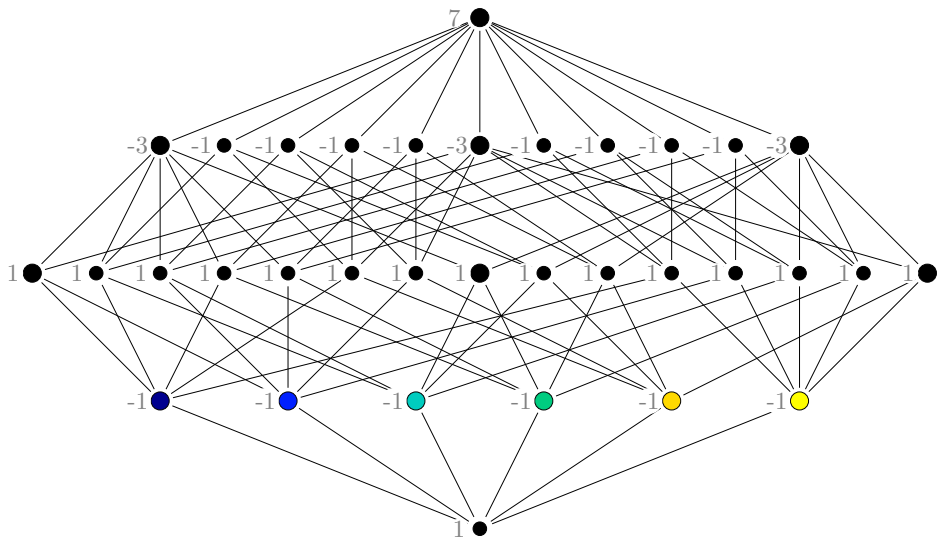 

Since the six boundary pieces~\eqref{Ws} are linear equations in log probability space, they define exponential families. 
For instance, the exponential family $\Lcal_{1,0}$ consists of all distributions whose log-probabilities have a vanishing inner product with $[1, -1, -1, 1, 0, 0, 0, 0]^\top$. A sufficient statistic is any set of vectors spanning the kernel of this vector.
Since intersections of exponential families are exponential families, each element in the intersection poset in Figure~\ref{figure:intersectionposet} 
is also an exponential family. 


\section{Maximum likelihood} \label{6}

In this section we give a closed-form formula for the maximum likelihood estimation to the model $\Mcal$. We also find the distributions whose divergence to the model is greatest.

Consider an empirical probability distribution coming from some data. The maximum likelihood estimation problem asks for the distribution in a statistical model with smallest Kullback-Leibler (KL) divergence to the data distribution. The KL divergence from $p$ to $q$ is defined as
$D(p\|q) := \sum_{x} p_x \log\frac{p_x}{q_x}$,
where $x$ ranges over the possible states of $p$ and $q$.
This is zero if and only if $p=q$ and it is set to $+\infty$ when $\supp(p)\not\subseteq \supp(q)$. 
The distributions in the closure of a model that minimize the KL divergence are called 
{\em reverse information projections} (rI-projections)~\cite{csiszar2004}.
In general they are not unique. 

\subsection{Reversed information projections}

To study the maximum likelihood estimation problem for the model $\mathcal{M}$, we first find the rI-projections to each boundary piece of the model. We use the description of the boundary pieces as exponential families from Section~\ref{5}. 
Proposition~\ref{slice} means we only need to consider projections onto the six boundary pieces, not onto the entire intersection poset (as we would have to for $\mathcal{M}_{3,2}$, see \cite{SerkanPaper}).
For a distribution $p \in \Delta_7 \backslash \mathcal{M}$, each rI-projection will lie on one of the boundary pieces, and there is at most one projection point in each boundary piece. 
Taking the projection that minimizes divergence, over the six boundary pieces, gives the rI-projection to the whole model. 

Let $\mathcal{P}_{i,j}$ be the toric hypersurface in the simplex obtained by exponentiating the hyperplane $\mathcal{L}_{i,j}$ in log-probability space and normalizing. The following proposition concerns maximum likelihood estimation for that toric model.

\begin{proposition}
	\label{proposition:rI-projection}
The unique rI-projection of $p \in \Delta_7$ onto $\mathcal{P}_{1,0}$, denoted $p_{\mathcal{P}_{1,0}}$, is found by taking the best rank one approximation in the slice $p_{0**}$ and leaving the other slice unchanged. In symbols,
\begin{equation*}
p_{\mathcal{P}_{1,0}}(X)=
\begin{cases}
p(X_2|X_1)p(X_3|X_1)p(X_1), &X_1=0\\
p(X), & X_1\not=0
\end{cases},
\end{equation*}
where $X$ is the random variable on state space $\{0,1\}^3$ and $X_i$ is its $i$th coordinate. The divergence from $p$ to $\mathcal{P}_{1,0}$ is 
\begin{equation*}
D(p\| \Pcal_{1,0}) = p(X_1=0) \cdot 
I_p(X_2;X_3|X_1=0),  
\end{equation*}
where $I_p(X_2;X_3|X_1=0) = D(p(X_2 X_3|X_1=0) \| p(X_2|X_1=0)p(X_3|X_1=0) )$ is the conditional mutual information of the two variables $X_2$ and $X_3$, given $X_1 = 0$. The rI-projections to the five other pieces follow analogously.
\end{proposition}
\begin{proof}
This follows applying \cite[Lemma~3.2]{NIPS2011_4380} to the exponential family described in Proposition~\ref{slice} and using the fact that the rI-projection of a distribution to an independence model is given by the product of its marginals. 
\end{proof}
 
 The distributions whose rI-projections to $\mathcal{P}_{1,0}$ coincide are those with the same values $p_{1**}$ and fixed marginals on $p_{0**}$. The rI-projection to the entire model is the boundary projection with smallest divergence value. It has divergence 
$$ D(p\| \Mcal) = \min_{i=1,2,3, \, \, j =0,1} D(p\| \Pcal_{i,j}) .$$
The rI-projection of any $p$ to an exponential family is unique, so there are at most six rI-projections to $\Mcal$. 

\begin{remark}
For the $\Mcal_{3,3}$ and $\RBM_{3,2}$ parametrizations of $\Mcal$, each rI-projection may be realized by several distinct choices of the parameters. 
This implies that there are several choices of parameters associated with each local maximizer of the likelihood function. 
\end{remark}

\subsection{Divergence maximizers}

The maximum divergence to a statistical model is a measure of the representational power of that model. 
The uniform distribution on the sets of vectors with even or odd parity need the maximum number of components to be arbitrarily well approximated by a mixture of products distribution (see~\cite{montufar2013mixture}).
Here, we show that these parity distributions have the largest divergence to the model $\Mcal$. 

\begin{proposition}
	\label{proposition:divergencemaximizers}
The maximum divergence to $\Mcal$ is $\frac{1}{2}\log 2$. 
The maximizers are 
$u^+ : = \frac{1}{4}(\delta_{000} + \delta_{011} + \delta_{101} + \delta_{110})$ and 
$u^- : = \frac{1}{4}(\delta_{001} + \delta_{010} + \delta_{100} + \delta_{111})$. 
There are six rI-projections of $u^+$, one in each boundary piece:\begin{gather*}
u^+_{\Pcal_{1,0}} = \frac{1}{8}(\delta_{000}+\delta_{001}+\delta_{010}+\delta_{011}) + \frac{1}{4}(\delta_{101}+\delta_{110}) \phantom{.}\\
u^+_{\Pcal_{1,1}} = \frac{1}{8}(\delta_{100}+\delta_{101}+\delta_{110}+\delta_{111}) + \frac{1}{4}(\delta_{011}+\delta_{000}) \phantom{.}\\
u^+_{\Pcal_{2,0}} = \frac{1}{8}(\delta_{000}+\delta_{001}+\delta_{100}+\delta_{101}) + \frac{1}{4}(\delta_{011}+\delta_{110}) \phantom{.}\\
u^+_{\Pcal_{2,1}} = \frac{1}{8}(\delta_{010}+\delta_{011}+\delta_{110}+\delta_{111}) + \frac{1}{4}(\delta_{000}+\delta_{101}) \phantom{.}\\
u^+_{\Pcal_{3,0}} = \frac{1}{8}(\delta_{000}+\delta_{010}+\delta_{100}+\delta_{110}) + \frac{1}{4}(\delta_{011}+\delta_{101}) \phantom{.}\\
u^+_{\Pcal_{3,1}} = \frac{1}{8}(\delta_{001}+\delta_{011}+\delta_{101}+\delta_{111}) + \frac{1}{4}(\delta_{000}+\delta_{110}) . 
\end{gather*} 
The projection points of $u^-$ are given in a similar way. 
\end{proposition}
\begin{proof}
Proposition~\ref{proposition:rI-projection} shows that the indicated distributions are the rI-projections of $u^+$ onto the individual boundary pieces of $\Mcal$. There can be no more than six projection points and hence we have a complete list. That $\frac{1}{2}\log 2$ is the maximum possible divergence to $\Mcal$ follows from upper bounds for mixtures of products and RBMs given in~\cite{Montufar2013}. 
Both $u^+$ and $u^-$ attain this upper bound. 

Now we show that $u^+$ and $u^-$ are the only divergence maximizers.
Assume without loss of generality that some maximizer $p$ has an rI-projection onto $\Mcal $ in $\Pcal_{1,0}$. 
Then 
$D(p\|\Pcal_{1,0}) = p(X_1=0) I_p(X_2; X_3|X_1=0) \leq D(p\|\Pcal_{1,1}) = p(X_1=1) I_p(X_2; X_3|X_1=1) \leq (1 - p(X_1=0))\log 2$. 
The last inequality follows since,
for two binary variables, the mutual information is maximized by a uniform distribution on strings of Hamming distance 2 (see~\cite{Ay2006}). 
The maximum value $\frac{1}{2}\log 2$ is attained only if $p(X_1=0)=p(X_1=1)=\frac{1}{2}$ and both $p(X_2 X_3|X_1=0)$ and $p(X_2 X_3|X_1=1)$ are uniform on pairs of Hamming distance 2. 
If these two conditional distributions were equal, then $p \in \Mcal$, and $p$ is not a divergence maximizer. Hence the pairs are different, and $p$ is a uniform distribution on $4$ strings of equal parity.\end{proof}

\begin{remark}
Proposition~\ref{proposition:divergencemaximizers} shows that the upper bound on the maximum divergence to mixtures
of products and RBMs from~\cite[Theorems~1 and~2]{Montufar2013} is tight in the case of $\Mcal_{3,3}$ and $\RBM_{3,2}$. 
Moreover it shows that for a given data point $\RBM_{3,2}$ can have up to $6$ global maximizers of the likelihood, and that generically this will be the number of local maximizers. 
\end{remark}

An interesting question is whether we can characterize the points in the probability simplex that project to the different boundary pieces of the model. 
That is, to provide a \emph{decision boundary} separating the regions of the simplex that are closer to each part of the model, with respect to the KL divergence. 
In our case, these decision boundaries are neither linear families nor exponential families. 

\section{Visualization in three dimensions} \label{7}

In \cite[Figure 3]{S}, a first attempt was made to visualize the model $\Mcal$. In this section, we explain how to draw the seven-dimensional model $\Mcal$ using a three dimensional figure. We make use of the following change of basis (corresponding to the basis of characters) in the log-probability coordinates: 

\begin{equation*}
\begin{pmatrix}
m_\emptyset\\
m_{\{3\}}\\
m_{\{2\}}\\
m_{\{2,3\}}\\
m_{\{1\}}\\
m_{\{1,3\}}\\
m_{\{1,2\}}\\
m_{\{1,2,3\}}
\end{pmatrix}
=
\begin{pmatrix*}[r]
1&    1&    1&    1&    1&    1&    1&    1\\
1&   -1&    1&   -1&    1&   -1&    1&   -1\\
1&    1&   -1&   -1&    1&    1&   -1&   -1\\
1&   -1&   -1&    1&    1&   -1&   -1&    1\\
1&    1&    1&    1&   -1&   -1&   -1&   -1\\
1&   -1&    1&   -1&   -1&    1&   -1&    1\\
1&    1&   -1&   -1&   -1&   -1&    1&    1\\
1&   -1&   -1&    1&   -1&    1&    1&   -1\\
\end{pmatrix*}
\begin{pmatrix}
l_{000} \\ l_{001} \\ l_{010} \\ l_{011} \\ l_{100} \\ l_{101} \\ l_{110} \\ l_{111} 
\end{pmatrix}.
\end{equation*}
The boundary pieces of the model can be written in terms of just four of these coordinates:
\begin{equation*}
\def\arraystretch{1.25}
\begin{array}{ll}
\Lcal_{1,0} =  \{m_{\{2,3\}} + m_{\{1,2,3\}} = 0 \}, &
\Lcal_{1,1} = \{ m_{\{2,3\}} - m_{\{1,2,3\}} = 0\}, \\
\Lcal_{2,0} = \{m_{\{1,3\}} + m_{\{1,2,3\}} = 0 \}, &
\Lcal_{2,1} = \{ m_{\{1,3\}} - m_{\{1,2,3\}} = 0\}, \\
\Lcal_{3,0} = \{m_{\{1,2\}} + m_{\{1,2,3\}} = 0 \}, &
\Lcal_{3,1} = \{ m_{\{1,2\}} - m_{\{1,2,3\}} = 0\} .
\end{array}
\label{eq:characterconstraints}
\end{equation*}
Hence it suffices to visualize the combinations of coordinates
$( m_{\{1, 2\}} ,m_{\{1, 3\}} , m_{\{2, 3\}} , m_{\{1, 2, 3\}}  )$
 that lie in the model. Furthermore, if a vector satisfies the inequalities above, then so does any scalar multiple of it.
This means we need consider only those $(m_{\{1, 2\}} ,m_{\{1, 3\}} , m_{\{2, 3\}}, m_{\{1,2,3\}})$ lying on the three-dimensional sphere. 
The value of $m_{\{1, 2, 3\}}$ can be found up to sign from the other three coordinates. 
We draw the model in coordinates
 \begin{equation}
(\overline{m}_{\{1,2\}} , \overline{m}_{\{1,3\}}, \overline{m}_{\{2,3\}}) = \frac{(m_{\{1,2\}} , m_{\{1,3\}}, m_{\{2,3\}})}{\|(m_{\{1,2\}}, m_{\{1,3\}}, m_{\{2,3\}}, m_{\{1,2,3\}})\|_2} ,
\label{eq:projectiviz}
\end{equation}
with separate panels for the different signs of $m_{\{1,2,3\}}$. Figure~\ref{2blobs} shows pieces $\Lcal_{1,0}$ and $\Lcal_{1,1}$.  The whole model is shown in Figure~\ref{figure:blobs}.
 
 \begin{figure}
 	\centering
 	\scalebox{.95}{	
 	\setlength{\unitlength}{1cm}
 	\setlength{\fboxsep}{0pt}%
 	\begin{picture}(17,8)
 	\put(0,0){\includegraphics[width=15.6cm]{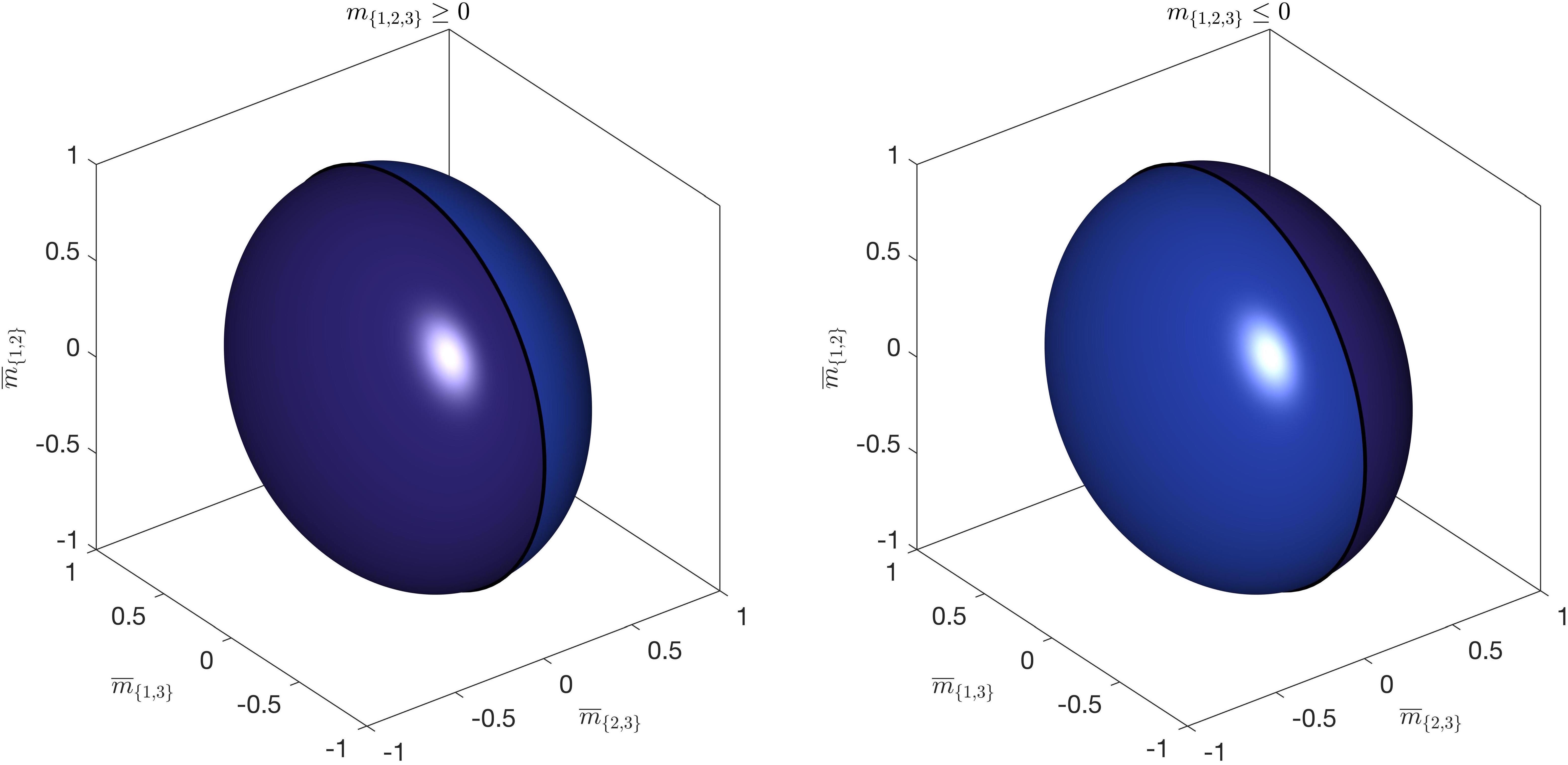}}
 	\put(15.4,2.5){\includegraphics[width=1.1cm]{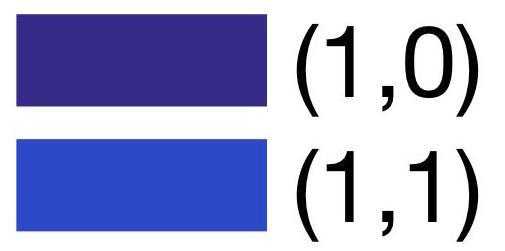}}
 	\end{picture}
 	}	
	\caption{Illustration of two boundary pieces of the model $\mathcal{M}$. The set $\Lcal_{1,0}$ is in dark blue, and $\Lcal_{1,1}$ is in light blue. The points enclosed by the surface correspond to distributions in the complement of the two basic semi-algebraic sets of $\RBM_{3,2}$ enclosed by $\Lcal_{1,0}$ and $\Lcal_{1,1}$. 
The black line is $\{ m_{\{2,3\}} = m_{\{1,2,3\}} = 0 \}$, along which $\Lcal_{1,0}$ and $\Lcal_{1,1}$ meet. The non-linearity of the surfaces is due to normalizing with respect to the $\|\cdot\|_2$ norm.}
	\label{2blobs}
  \end{figure}
 
 \begin{figure}
	\centering
\scalebox{.95}{	
	\setlength{\unitlength}{1cm}
	\setlength{\fboxsep}{0pt}%
	\begin{picture}(17,8)
	\put(0,0){\includegraphics[width=15.6cm]{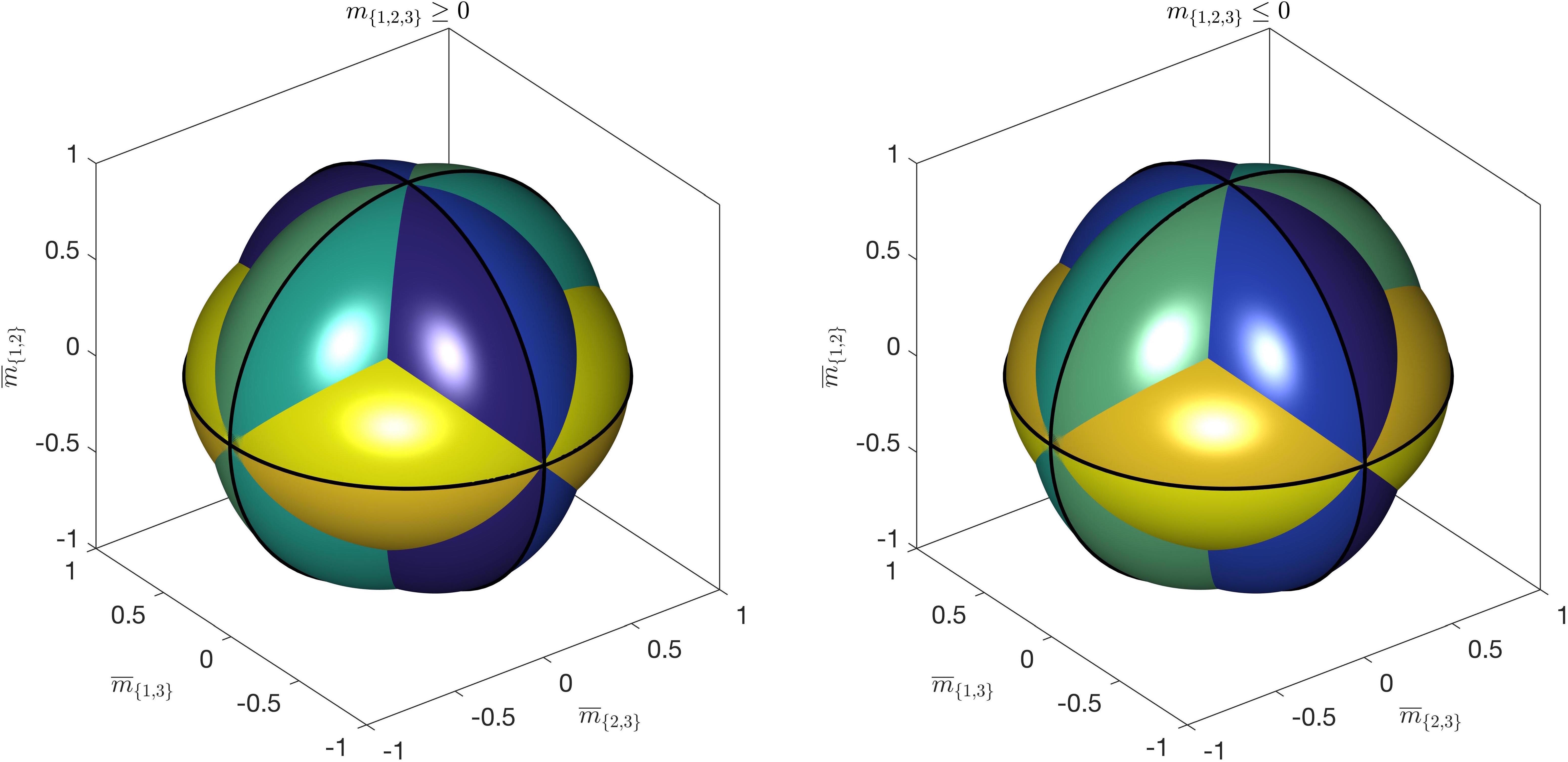}}
	\put(15.4,2.5){\includegraphics[width=1.1cm]{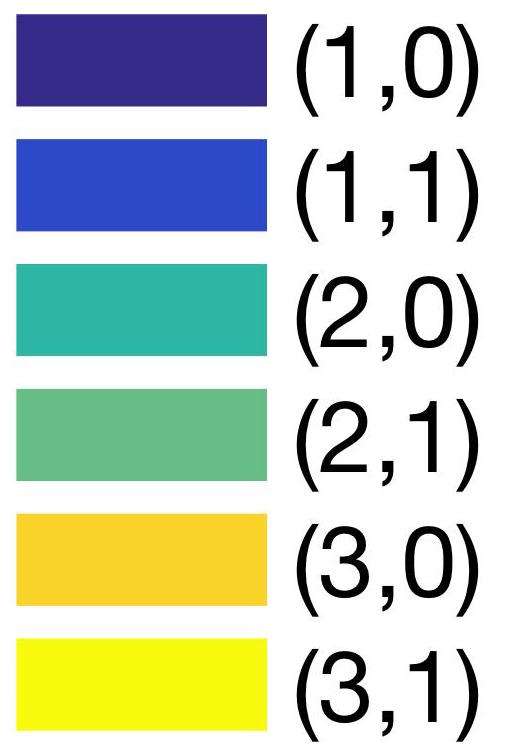}}
	\end{picture}
}	
	\caption{Illustration of $\Mcal$ in the~\eqref{eq:projectiviz} coordinates. 
	The model occupies the space inside the three-sphere that is outside any of the blue, green, or yellow surfaces. The colours correspond to the six boundary pieces of the model. 
	Within each orthant, the part of the sphere outside all three surfaces is a triangular bipyramid. Four of these make up the model $\Mcal_{3,2}$. }
	\label{figure:blobs}
\end{figure}

\section{Outlook}
\label{section:conclusions}

We proved the rather surprising fact that a mixture of products and a product of mixtures represent the same set of probability distributions. 
Although for larger models this is known not to be true in general~\cite{montufar2015does}, it points at a close similarity of both models. 

In most previous work on the representational power of RBMs, membership in the model is determined by constructing parameters that realize certain probability distributions. 
In contrast, the implicit descriptions discussed here fully characterize distributions that are in the model.
As we have shown, the semi-algebraic description also allows the computation of maximum likelihood estimates and divergence maximizers, both of which appear quite difficult to obtain via other methods.

The natural next step is to extend the analysis to larger models. 
However, the description for larger models involves complicated equality constraints. For example, in~\cite{Cueto:2010:ICB:1866469.1866627} the Zariski closure of the model $\RBM_{4,2}$ is found. It is the zero set of a single degree 110 polynomial with at least 17,214,912 terms.  The binomial inequalities we obtain here are more tractable.

In light of this, it appears natural to consider approximate descriptions of larger RBM models in terms of inequality constraints only. A relaxation of larger statistical models, given in terms of inequalities only, would provide lower bounds on the maximal divergence and the minimal size of universal approximators.  

In~\cite{ALLMAN201537} the authors show that the model $\Mcal_{n,2}$ consists of supermodular distributions with flattening rank at most two. Distributions in larger RBM models are Hadamard products of non-negative tensors of rank at most two (products of tensors proportional to distributions in $\Mcal_{n,2}$). 
Ignoring the equations, we have the set of supermodular tensors, which consists of basic semi-algebraic sets satisfying binomial quadratic inequalities as in~\eqref{ineq}. 
Hence the algebraic boundary of Hadamard products of supermodular tensors is again a union of exponential families,
for 
which we may hope to obtain maximum likelihood estimates in closed form. 

\bigskip

{\bf Acknowledgments.} We are grateful to Bernd Sturmfels for fruitful discussions. 
GM acknowledges support from the Erwin Schr\"odinger Institute.

\let\OLDthebibliography\thebibliography
\renewcommand\thebibliography[1]{
  \OLDthebibliography{#1}
  \setlength{\parskip}{0pt}
  \setlength{\itemsep}{2pt}
}
\begin{small}
\bibliographystyle{abbrv}
\bibliography{referenzen}
\end{small}

\vspace{-.1cm}
\noindent \footnotesize {\bf Authors' addresses:}

\noindent Anna Seigal,
University of California, Berkeley, USA,
{\tt seigal@berkeley.edu}.

\noindent Guido Mont\'ufar, 
Max Planck Institute for Mathematics in the Sciences, Leipzig, Germany; and Departments of Mathematics and Statistics, University of California, Los Angeles, USA, 
{\tt montufar@mis.mpg.de}. 

\end{document}